\documentclass{article}

\usepackage{arxiv}

\usepackage[utf8]{inputenc} 
\usepackage[T1]{fontenc}    
\usepackage{hyperref}       
\usepackage{url}            
\usepackage{booktabs}       
\usepackage{amsfonts}       
\usepackage{nicefrac}       
\usepackage{microtype}      
\usepackage{lipsum}
\usepackage{graphicx}

\usepackage{amsmath}
\usepackage{amssymb}
\usepackage{mathtools}
\usepackage{todonotes}
\usepackage{cleveref}
\usepackage{algorithm}
\usepackage[noend]{algorithmic}

\DeclareMathOperator*{\argmin}{arg\,min}

\newtheorem{theorem}{Theorem}
\newtheorem{proof}{Proof}

\title{An Analysis of Regularized Approaches for Constrained Machine Learning}

\author{
     Michele Lombardi \\
     DISI, ALMA-AI\\
     University of Bologna\\
     \texttt{michele.lombardi2@unibo.it} \\
     \And
     Federico Baldo \\
     DISI\\
     University of Bologna\\
     \texttt{federico.baldo2@unibo.it} \\
     \And
     Andrea Borghesi \\
     DISI, ALMA-AI\\
     University of Bologna\\
     \texttt{andrea.borghesi3@unibo.it} \\
     \And
     Michela Milano \\
     DISI, ALMA-AI\\
     University of Bologna\\
     \texttt{michela.milano@unibo.it} \\
}

\date{}

\begin{document}

\maketitle

\section{Context}

Regularization-based approaches for injecting constraints in Machine Learning
(ML) were introduced (see e.g. \cite{DBLP:journals/ai/DiligentiGS17}) to improve
a predictive model via expert knowledge. Given the recent interest in ethical
and trustworthy AI, however, several works are resorting to these approaches for
enforcing desired properties over a ML model (e.g. fairness
\cite{DBLP:conf/aaai/AghaeiAV19,zemel2013learning,calders2010three}).
Regularized approaches for constraint injection solve, in an exact or
approximate fashion, a problem in the form:
\begin{align}
    &\argmin_{w \in W} \{ L(y) + \lambda^\top C(y) \} \quad \text{with: } y = f(\mathbf{x}; w) \label{eq:reg}
\end{align}
where $L$ is a loss function and $f$ is the model to be trained, with parameter
vector $w$ from a parameter space $W$. We use $f({\bf x}; w)$ to refer to the
model output for the whole training set $\bf x$.  The regularization function
$C$ denotes a vector of (non-negative) constraint violation indices for $m$
constraints, while $\lambda \geq 0$ is a vector of weights (or
\emph{multipliers}).

As an example, in a regression problem we may desire a specific output ordering
for two input vectors in the training set. A viable regularizer may be:
\begin{equation}
    C(y) \equiv \max(0, y_i - y_j) \label{eq:regex}
\end{equation}
the term is zero iff the constraint $y_i \leq y_j$ is satisfied.  For obtaining
balanced predictions in a binary classification problem, we may use instead:
\begin{equation}
    C(y) \equiv \left|\sum_{i = 1}^n y_{i} -  \frac{n}{2}\right|
    \label{eq:classexample}
\end{equation}
where $y_i$ is the binary output associated to one of the two classes. If $n$ is
even, the term is 0 for perfectly balanced classifications.

When regularized methods are used to enforce constraints, \emph{a typical
approach consists in adjusting the $\lambda$ vector until a suitable compromise
between accuracy and constraint satisfaction is reached} (e.g. a discrimination
index becomes sufficiently low). This approach enables the use of traditional
training algorithms, at the cost of having to search over the space of possible
multipliers.

Though the method is known to work well in many practical cases, the process has
been subject to little general analysis. With this note, we aim to make a
preliminary step in this direction, providing a more systematic overview of the
strengths and (in particular) potential weaknesses of this class of approaches.

\section{Analysis}%
\label{sec:Analysis}

Regularized approaches for constraint injection are strongly related to duality
in optimization, from which many of the results we report can be derived.
Despite this, we present an analysis based on first principles and tailored to
our use case, as it provides additional insights. It will be convenient to
reformulate \Cref{eq:reg} by embedding the ML model structure in the $L$ and $C$
functions:
\begin{align}
    {\rm\bf PR(\theta):}\ &\ \argmin_{w \in W} \{ L(w) + \lambda^\top C(w) \} \label{eq:pr}
\end{align}
With some abuse of notation $L(w)$ refers to $L(f({\bf x}; w))$, and the same
for $C(w)$. This approach enables a uniform treatment of convex and non-convex
models or functions. We are interested in the relation between the unconstrained
PR formulation and the following constrained training problem:
\begin{align}
    {\rm\bf PC(\lambda):}\ &\ \argmin_{w \in W} \{ L(w) \mid C(w) \leq \theta \} \label{eq:pc}
\end{align}
where $\theta$ is a vector of thresholds for the constraint violation indices.
In ethical or trustworthy AI applications, PC will be the most natural problem
formulation.

We wish to understand \emph{the viability of solving PC indirectly, by adjusting
the $\lambda$ vector and solving the unconstrained problem PR}, as depicted in
\Cref{alg:pr4pc}; line 2 refers to some kind of search over
the multiplier space.
Ideally, the algorithm should be equivalent to solving the PC formulation
directly. For this to be true, solving PR($\lambda$) should have a chance to
yield assignments that are optimal for the constrained problem. Moreover, an
optimum of PC($\theta$) should always be attainable in this fashion. Additional
properties may enable more efficient search. In the note, we will characterize
\Cref{alg:pr4pc} to the best of our abilities.

\begin{algorithm}[tb]
\caption{\sc pr4pc($\theta$)}
\begin{algorithmic}[1]
    \FOR{$\lambda \in (\mathbb{R}^+)^m$}
    \STATE Optimize PR to find $w^*$
    \IF{$C(w^*) \leq \theta$}
    \STATE Store $w^*, L(w^*)$
    \ENDIF
    \ENDFOR
    \STATE Pick the stored solution with the smallest $L(w^*)$
\end{algorithmic}
\label{alg:pr4pc}
\end{algorithm}

\paragraph{Regularized and Constrained Optima}%

The relation between the PR and PC formulations are tied to the properties of
their optimal solutions. An optimal PC solution $w^*_c$ satisfies:
\begin{align}
    {\bf opt_{c}(w^*, \theta):}\ &\ L(w) \geq L(w^*) \quad \forall w \in W \mid C(w) \leq \theta  
\end{align}
while for an optimal solution $w^*_r$ of PR with multipliers $\lambda$ we have:
\begin{align}
    {\bf opt_{r}(w^*, \lambda):}\ &\ L(w) + \lambda^\top C(w) \geq L(w^*) + \lambda^\top C(w^*) \quad \forall w \in W
\end{align}
The definitions apply also to local optima, by swapping $W$ with some
neighborhood of $w^*_c$ and $w^*_r$. We can now provide the following result:

\begin{theorem}\label{thr:pr2pc}
an optimal solution $w^*$ for PR is also optimal for PC, for a threshold equal
to $C(w^*)$:
\begin{equation}
    opt_r(w^*, \lambda) \Rightarrow opt_c(w^*, C(w^*)) \label{eq:fwd}
\end{equation}
\end{theorem}

\begin{proof}[by contradiction]
Let us assume that $w^*$ is an optimal solution for PR but not optimal for PC,
i.e. that there is a feasible $w^\prime \in W$ such that:
\begin{equation}
    L(w^\prime) < L(w^*) \label{eq:prem}
\end{equation}
Since $w^*$ is optimal for PR, we have that:
\begin{equation}
    L(w^\prime) \geq L(w^*) + \lambda^\top (C(w^*) - C(w^\prime)) \label{eq:cons}
\end{equation}
Since $w^\prime$ is feasible for $\theta = C(w^*)$, we have that its violation
vector cannot be greater than that of $w^*$. Formally, we have that $C(w^\prime)
\leq C(w^*)$, or equivalently $C(w^*) - C(w^\prime) \geq 0$. Therefore
\Cref{eq:cons} contradicts \Cref{eq:prem}, thus proving the original point. The
same reasoning applies to local optima. \qed 
\end{proof}

Theorem~\ref{thr:pr2pc} shows that solving PR($\lambda$) \emph{always results in
an optimum for the constrained formulation}, albeit for threshold $\theta =
C(w^*)$ that cannot be a priori chosen. The statement is true even for
non-convex loss, reguralizer, and model structure. This is a simple, but
powerful result, which provides a strong motivation for
Algorithm~\ref{alg:pr4pc}. 

\paragraph{Global vs Local Optimality}

If regularized problems can be solved to global optimality, then increasing a
weight in the $\lambda$ vector cannot have an adverse effect on the satisfaction
level of the corresponding constraint. Formally, there is a monotonic relation
between $\lambda$ and $C(w^*)$:
\begin{equation}
    opt_r(w^\prime, \lambda^\prime), opt_r(w^{\prime\prime}, \lambda^{\prime\prime}), \lambda_j^\prime \geq \lambda_j^{\prime\prime} \Rightarrow C_j(w^\prime) \leq C_j(w^{\prime\prime})
\end{equation}
The proof is omitted due to lack of space. When monotonicity holds, searching
over the multiplier space in \Cref{alg:pr4pc} can be considerably simpler (e.g.
binary search for a single multiplier, or sub-gradient descent in general
\cite{fioretto2020lagrangian}).

However, global optimality is attainable only in very specific cases (e.g.
convex loss, regularizer, and model) or by solving PR in an exact fashion (which
may be computationally expensive). Failing this, monotonicity will not strictly
hold, in the worst case requiring exhaustive (or semi-exhaustive) search on the
multiplier space. Additionally, relying on local optima will lead to suboptimal
solutions (subject to uncertainty if stochastic training algorithm is employed).

\paragraph{Unique vs Multiple Optima}

Further issues arise (even for global optimality) when the regularized problem
PR($\lambda$) has multiple equivalent optima. In the fully convex case, this may
happen if the multiplier values cause the presence of plateaus (see
\Cref{fig:plateau}A, where $\lambda = 1$). In the (more practically relevant)
non-convex case, there may be separate optima with the same value for the
regularized loss, but different trade-offs between loss and constraint
violation: this is depicted for a simple example in \Cref{fig:plateau}B.

Multiple equivalent optima may cause a non-monotonic relation between $\lambda$
and the constraint satisfaction level, similarly to what discussed in the
previous paragraph.

Additionally, it may happen that different constrained optima are associated to
the same multiplier, \emph{and to no other multiplier}.  In \Cref{fig:plateau}A,
for example, the multiplier $\lambda = 1$ is associated to all optimal solutions
of PC($\theta$) with $\theta \leq \theta^*$; no other multiplier is associated
to the same solutions. Unless some kind of tie breaking technique is employed,
this situation makes specific constrained optima impossible to reach.

\begin{figure}[tb]
\begin{center}
    \includegraphics[width=.9\textwidth]{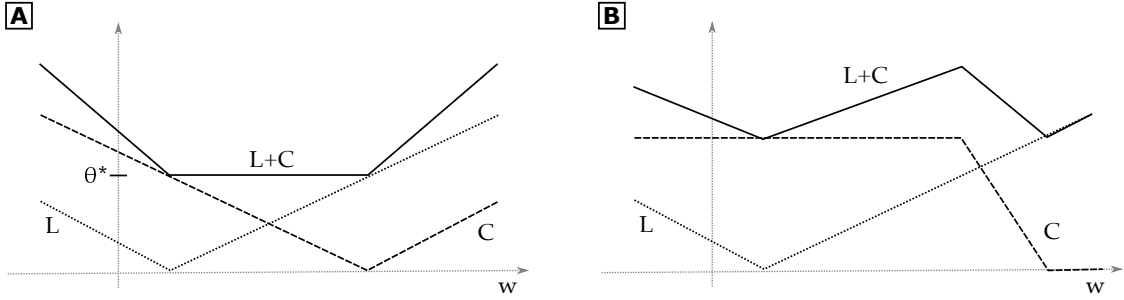}
\end{center}
\caption{Multiple Optima in Convex (A) and Non-Convex (B) Regularized Problems}
\label{fig:plateau}
\end{figure}

\paragraph{Inaccessible Constrained Optima}

We next proceed to investigate whether an optimum of the constrained formulation
may be associated to no multiplier value: any such point would be completely
unattainable via \Cref{alg:pr4pc}. We have that:
\begin{theorem}\label{thr:pc2pr} 
    An optimal solution $w^*$ for PC is optimal for
    PR iff there exists a multiplier vector $\lambda$ that satisfies:
\begin{equation}
    \max_{\mathclap{\substack{w \in W, \\ C_j(w) > C_j(w^*)}}}\ R(w, \lambda) \leq \lambda_j \leq \min_{\mathclap{\substack{w \in W, \\ C_j(w) < C_j(w^*)}}}\ R(w, \lambda)
\end{equation}
with:
\begin{equation}
    R(w, \lambda) = - \frac{\Delta L(w, w^*) + \lambda_{\overline{j}}^\top \Delta C_{\overline{j}} (w, w^*) 
    }{\Delta C_j(w, w^*)} \label{eq:boundexp}
\end{equation}
\end{theorem}

In the theorem, we refer with $\Delta C(w, w^*)$ to the difference $C(w) -
C(w^*)$ and with $\Delta L(w, w^*)$ to the difference $L(w) - L(w^*)$. Moreover,
$\overline{j}$ refers to the set of all multiplier indices, except for $j$.
Intuitively, every assignment for which constraint $j$ has a lower degree of
violation than in $w^*$ enforces an upper bound on $\lambda_j$; every assignment
for which the violation is higher enforces a lower bound.

\begin{proof}
Let $w^*$ be a PC optimum for some threshold $\theta$; this implies that $w^*$
is also optimal for a tightened threshold, i.e. for $\theta = C(w^*)$. We
therefore have:
\begin{equation}
    L(w) \geq L(w^*) \quad \forall w \in W, C(w) \leq C(w^*) \label{eq:pcoptt_prem}
\end{equation}
We are interested in the conditions for $w^*$ to be optimal for the regularized
formulation, for some multiplier vector $\lambda$. This is true iff:
\begin{equation}
    L(w) + \lambda^\top C(w) \geq L(w^*) + \lambda^\top C(w^*) \quad \forall w \in W 
\end{equation}
which can rewritten as:
\begin{equation}
    \lambda^\top \Delta C(w, w^*) + \Delta L(w, w^*) \geq 0 \quad \forall w \in W  \label{eq:propt_cond}
\end{equation}
If $\Delta C(w, w^*) = 0$, then \Cref{eq:propt_cond} is trivially satisfied for
every multiplier vector, due to \Cref{eq:pcoptt_prem}. Otherwise, at least some
component in $\Delta C(w, w^*)$ will be non-null, so that we can write:
\begin{equation}
    \lambda_j \Delta C_j(w, w^*) + \lambda_{\overline{j}}^\top \Delta C_{\overline{j}} (w, w^*) + \Delta L(w, w^*) \geq 0
\end{equation}
If $\Delta C_j(w,w^*) < 0$, we get:
\begin{equation}
    \lambda_j \leq - \frac{\Delta L(w, w^*) + \lambda_{\overline{j}}^\top \Delta C_{\overline{j}} (w, w^*) 
    }{\Delta C_j(w, w^*)} \quad \forall w \in W \mid C_j(w) < C_j(w^*)
\end{equation}
I.e. a series of upper bounds for $\lambda_j$. If $\Delta C_j(w,w^*) > 0$, we get:
\begin{equation}
    \lambda_j \geq - \frac{\Delta L(w, w^*) + \lambda_{\overline{j}}^\top \Delta C_{\overline{j}} (w, w^*) 
    }{\Delta C_j(w, w^*)} \quad \forall w \in W \mid C_j(w) > C_j(w^*)
\end{equation}
I.e. a series of lower bounds on $\lambda_j$. From these the original result is obtained. \qed
\end{proof}

The main consequence of Theorem~\ref{thr:pc2pr} is that the reported system of
inequalities may actually admit no solution, meaning that \emph{some constrained
optima may be unattainable} via \Cref{alg:pr4pc}. This is the case for the
optimum $w^*$ (for threshold $\theta^*$) in the simple example from
\Cref{fig:forbidden}, since any multiplier value will result in an unbounded
regularized problem. This is a potentially serious limitation of regularized
methods: the actual severity of the issue will depend on the specific properties
of the loss, regularizer, and ML model being considered.

\begin{figure}[tb]
\begin{center}
    \includegraphics[width=.9\textwidth]{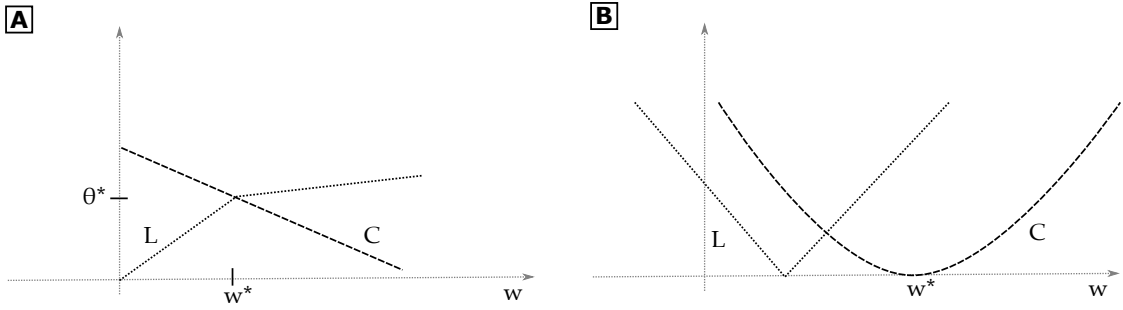}
\end{center}
\caption{(A) Unattainable Constrained Optimum; (B) Numerical Issues for $w^*$}
\label{fig:forbidden}
\end{figure}

\paragraph{Numerical Issues}

Theorem~\ref{thr:pc2pr} highlights another potential issue of regularized
approaches, arising when assignments with constraint violations arbitrarily
close to $C(w^*)$ exist. In such a situation, the denominator in
\Cref{eq:boundexp} becomes vanishingly small: depending on the properties of the
loss function, this may result in arbitrarily high lower bounds or arbitrarily
small upper bounds.  Informally, reaching a specific optimum for the constrained
problem may require \emph{extremely high or extremely low multipliers}, which
may cause numerical issues at training time. A simple example is depicted in
\Cref{fig:forbidden}B, where a regularizer with vanishing gradient and a loss
with non-vanishing gradient are combined. In such a situation, the constrained
optimum $w^*$ is reached via \Cref{alg:pr4pc} only for $\lambda \rightarrow
\infty$.

\paragraph{Differentiability}
Besides the ones reported here, one should be wary of pitfalls that are not
immediately related to \Cref{alg:pr4pc}. Many regularization based approaches
for constraint injection, for example, require differentiability of the $C$
function, which is often obtained by making approximations. For instance, in
\Cref{eq:classexample} differentiability does not hold due to the use of binary
variables; relaxing the integrally constraint address the issue, but allows to
satisfy the constraints by assigning 0.5 to all outputs, i.e. by having
completely uncertain, rather than balanced, predictions.

\section{Conclusions}%
\label{sec:Conclusions}

Combining the ML and optimization paradigms is a very interesting research
avenue still under ongoing exploration by the AI community.  Integrating
learning and optimization will lead to approaches better suited for ethical and
trustworthy AI (e.g. by making sub-symbolic models fair and explainable).  A
possible method to merge these paradigm consists in adding a regularization term
to the loss of a learner, to constrain its behaviour.  In this note, we offered
a preliminary discussion on a particular aspect of this problem, namely we
tackle the issue of finding the right balance between the loss (the accuracy of
the learner) and the regularization term (the degree of constraint
satisfaction); typically, this search is performed by adjusting a set of
multipliers until the desired compromise is reached.
The key results of this paper is the formal demonstration that this type of
approach, albeit well suited for many practical circumstances, \emph{cannot
guarantee to find all optimal solutions}.  In particular, in the non-convex case
there might be optima for the constrained problem that do not correspond to any
multiplier value. This result clearly hinders the applicability of
regularizer-based methods, at least unless more research effort is devoted to
discover new formulations or algorithms.

\bibliographystyle{abbrv}
\bibliography{main}

\end{document}